\documentclass{article}

\usepackage{PRIMEarxiv}

\usepackage[utf8]{inputenc} 
\usepackage[T1]{fontenc}    
\usepackage{hyperref}       
\usepackage{url}            
\usepackage{booktabs}       
\usepackage{amsfonts}       
\usepackage{nicefrac}       
\usepackage{microtype}      
\usepackage{lipsum}
\usepackage{fancyhdr}       
\usepackage{graphicx}       
\graphicspath{{media/}}     
\usepackage{amssymb}
\usepackage{latexsym}

\usepackage{graphicx}
\usepackage{mathtools}
\usepackage{amsfonts}
\usepackage{pifont}
\usepackage{booktabs}
\usepackage{todonotes}
\usepackage{siunitx}
\usepackage{microtype}
\usepackage{dirtytalk}
\usepackage{amsthm}
\usepackage{makecell} 
\usepackage{xcolor}
\usepackage{natbib}
\usepackage{algorithm}%
\usepackage{algorithmicx}%
\usepackage{algpseudocode}%
\usepackage{subcaption}
\usepackage[symbol]{footmisc}

\DeclarePairedDelimiterX{\inner}[2]{\langle}{\rangle}{#1, #2}

\newtheorem{definition}{Definition}

\newtheorem{theorem}{Theorem}

\newtheorem{corollary}{Corollary}

\newcounter{GittaCounter}

\newcounter{HolgerCounter}

\newcounter{PhilippCounter}

\newcommand*\widebar[1]{\@ifnextchar^{{\wide@bar{#1}{0}}}{\wide@bar{#1}{1}}}

\title{Robust identifiability for symbolic recovery of differential equations
}

\author{
  Hillary Hauger* \\
  Ludwig-Maximilians-Universit\"at M\"unchen \\
  Munich \\
  Germany\\
  \texttt{hillary.hauger@yahoo.com} \\
   \And
  Philipp Scholl* \\
  Ludwig-Maximilians-Universit\"at M\"unchen \\
  Munich \\
  Germany\\
  \texttt{scholl@math.lmu.de} \\
   \And
  Gitta Kutyniok \\
  Ludwig-Maximilians-Universit\"at M\"unchen \\
  University of Troms\o{} \\
  DLR-German Aerospace Center \\
  Munich Center for Machine Learning (MCML) \\
  Munich \\
  Germany\\
  \texttt{kuytniok@math.lmu.de} \\
}

\begin{document}
\maketitle

\begin{abstract}
Recent advancements in machine learning have transformed the discovery of physical laws, moving from manual derivation to data-driven methods that simultaneously learn both the structure and parameters of governing equations. This shift introduces new challenges regarding the validity of the discovered equations, particularly concerning their uniqueness and, hence, identifiability. While the issue of non-uniqueness has been well-studied in the context of parameter estimation, it remains underexplored for algorithms that recover both structure and parameters simultaneously. Early studies have primarily focused on idealized scenarios with perfect, noise-free data. In contrast, this paper investigates how noise influences the uniqueness and identifiability of physical laws governed by partial differential equations (PDEs). We develop a comprehensive mathematical framework to analyze the uniqueness of PDEs in the presence of noise and introduce new algorithms that account for noise, providing thresholds to assess uniqueness and identifying situations where excessive noise hinders reliable conclusions. Numerical experiments demonstrate the effectiveness of these algorithms in detecting uniqueness despite the presence of noise.
\end{abstract}


\section{Introduction}

\let\thefootnote\relax\footnotetext{* Equal contribution}
Motivated by the success of symbolic regression algorithms \citep{schmidt2009lipson, xu2020, petersen2019deep, udrescu2020ai, scholl2023parfam, cranmer2023interpretable}, which aim to learn interpretable, symbolic models for regression problems, there has been an increasing interest in learning governing equations symbolically, without predefined parametric models \citep{bongard2007automated, Brunton2016, Rudy2017DatadrivenDO, kaheman2020sindy, chen2021physics, Hasan2020}. To reliably determine the governing equation describing the data, it is crucial to identify whether the data corresponds to a unique governing equation since it is non-identifiable otherwise. This has been extensively studied for classical parameter estimation methods \citep{pohjanpalo1978system, VAJDA1984, cobelli1980parameter, walter1982global}.\citet{scholl2023icassp, scholl2023welldefinedness} introduced a mathematical framework for determining the identifiability for wide classes of differential equations, necessary for the symbolic recovery of differential equations. However, this framework does not address the influence of noise. Since real-life data is almost always contaminated with noise, it is crucial to understand its impact on the identifiability of physical laws. Therefore, we introduce extensions of the theoretical results from \citet{scholl2023icassp, scholl2023welldefinedness} which incorporate noise and yield thresholds that we use for new algorithms that can be applied in practice to infer the uniqueness and, therefore, the identifiability of a learned differential equation.

\section{Background}
We start with defining the uniqueness of a differential equation.
\begin{definition}[Uniqueness \citep{scholl2023icassp, scholl2023welldefinedness}]
 Let $u: U \rightarrow \mathbb{R}$ be a differentiable function on the open set $U \subset\mathbb{R}^k$. Let each
$g_1, ..., g_k: U \rightarrow \mathbb{R}$ be either a projection on one of the coordinates, any derivative of u that exists, or the function $u$.
Denote $G = (g_1, ..., g_k) : U \rightarrow \mathbb{R}^k$. Let $V$ be a set of functions which map from $ \mathbb{R}^{k}$ to $ \mathbb{R}$ and $F \in V$ such that for all $(t,x) \in U$
\begin{equation}
    \frac{\partial u}{\partial t}(t,x) = F(G(t,x)).
    \label{eq:pde_def}
\end{equation}
We say that the function u solves a unique PDE described by functions in $V$ for $G$  if F is the unique function in $V$ such that Equation \ref{eq:pde_def} holds.
\end{definition}
In the following, we will consider three different kinds of PDEs: We call a PDE linear if $F$ in Equation \ref{eq:pde_def} is linear, algebraic if $F$ is an algebraic function, and analytic if $F$ is an analytic function.
\cite{scholl2023icassp, scholl2023welldefinedness} proved multiple conditions for the uniqueness of a PDE. In practice, the most important conditions are: (1) a linear PDE is unique iff $G$ has full rank (i.e., the coordinates $g_1,...,g_k$ are linearly independent); (2) an algebraic PDE is unique for an algebraic solution $u$ iff there is at least one point $(t,x)\in\mathbb{R}^{m+1}$ such that $J_G(t,x)$ has full rank; (3) an analytic PDE is unique for a continuously differentiable solution $u$ iff there is at least one point $(t,x)\in\mathbb{R}^{m+1}$ such that $J_G(t,x)$ has full rank. Based on these results, they introduced two algorithms to identify uniqueness: Stable Feature Rank Computation (S-FRanCo) (for linear PDEs) and Jacobian Rank Computation (JRC) (for algebraic and analytic PDEs), which utilize finite differences and singular value decomposition to assess the rank of $G$ and $J_G$.
However, these algorithms lack mathematical guarantees on noisy data. In the following, we theoretically analyze the effect of noise on identifying uniqueness and use the new insights to derive new algorithms for determining identifiability.
\section{Methods}
\label{sec:methods}
To understand the impact of noise on identifying the uniqueness of a PDE, we first examine its influence on the basic methods, singular values (SVs) and finite differences (FD), used in S-FRanCo and JRC \citep{scholl2023icassp, scholl2023welldefinedness} in Subsection~\ref{sec:svs-fd}. We then derive new bounds for identifying the uniqueness of PDEs for noisy functions $u$ in Subsections~\ref{sec:nr-franco} and \ref{sec:nr-jrc}.


\subsection{Singular Values and Finite Differences} \label{sec:svs-fd}
We first cite the following theorem which bounds the influence of noise on $\sigma_k(A)$, the $k$-th SV of $A \in \mathbb{R}^{m \times n}$.
\begin{theorem}[\citep{Weyl_1912, Horn_Johnson_2012}] 
\label{theorem:mirksy}
If $A,E \in \mathbb{R}^{m \times n}$ are two arbitrary matrices, then
$|\sigma_k(A+E) - \sigma_k(A)| \leq \|E\|_F$ for all  $k=1,...,\min\{m, n\}$.
\end{theorem}
As the next step, we derive error bounds for the calculation of the derivative via the central FD formula. Consider an $l$ times continuously differentiable function $u:~\mathbb{R} \rightarrow~ \mathbb{R}$. In this section, we only consider one-dimensional input, but this can easily be extended to multiple dimensions by fixing the other coordinates. 
In the following let $x_k = x_0+kh, k = 1,...,n,$ denote the equispaced data points for $x_0 \in \mathbb{R}$, $h>0$, and even FD order $n\in\mathbb{N}$. Let $\tilde{u}(x)=u(x)+e(x)$ with bounded noise $e:\mathbb{R}\rightarrow\mathbb{R}$, $\|e\|_{\infty}<\epsilon$, and let $p_n(u,x) = \sum_{k=0}^nu(x_k)L_{n,k}(x)$, where $L_{n,k}(x) =\prod_{i=0,i\neq k}^n \frac{x-x_i}{x_k-x_i}$ are the Lagrange coefficients~\citep{jordan1960calculus}. Then, $p_n^{(l)}(u,x_{n/2})$ is the $l^{th}$ derivative computed with central FD order $n$ at $x_{n/2}$. 
To compute the distance between $u^{(l)}$ and $p_n^{(l)}(\tilde{u},x)$, we utilize the approximation bounds on Lagrange polynomials $u(x) = p(u,x) + r_n(u,x)$~\citep{Burden1989},
where $r_n(u, x)=\frac{u^{(n+1)}(\xi (x))}{(n+1)!}\prod_{k=0}^n(x-x_k)$, for some $x \in [x_0,x_n] $ and $\xi(x)\in [x_0,x_n]$. This yields the bound
\begin{align}
    | u^{(l)}(x) - p_n^{(l)}(\tilde{u},x) | \\
     \leq | u^{(l)}(x) - p_n^{(l)}(u,x) | + |p_n^{(l)}(e,x) | \\
    \leq |r_n^{(l)}(u,x)| + \epsilon\sum_{k=0}^n|L_{n,k}^{(l)}(x)| \text{.}
    \label{eq:fd_error_bound}
\end{align}
Note that $|r_n^{(l)}(u,x)|  = O(h^n)$ and $\epsilon\sum_{k=0}^n|L_{n,k}^{(l)}(x)| = O ( h^{-l})$ for $x\in \{x_k\}_{k=1}^n$ and $h\rightarrow 0$ \citep{Gautschi_2012}, showing that the error due to FD converges to 0 for small $h$, while the error due to the noise $e$ converges to infinity. By bounding $|r_n^{(l)}(u,x_{n/2})|$ and $|L_{n,k}^{(l)}(u,x_{n/2})|$ we can upper bound the error caused by FD on $u^{(l)}$ with some function $e_{fd}$ depending on $u,l,n,h,\epsilon$:
 \begin{equation}
    \| u^{(l)}(\cdot) - p_n^{(l)}(\tilde{u},\cdot) \|_{\infty} \leq e_{fd}(u,l,n,h,\epsilon).
\end{equation}
The following shows $e_{fd}$ for the first derivative of $u$:
\begin{align}
    e_{fd}(u,1,n,h,\epsilon) &= h^n \frac{ \|u^{n+1}\|_{\infty} (\frac{n}{2}!)^2}{(n+1)!}\\ 
    &+ \frac{\epsilon}{h} \sum_{k=0,k\neq\frac{n}{2}} \frac{\prod_{i=0,i\neq k,\frac{n}{2}}^n (\frac{n}{2}-i)}{\prod_{i=0,i\neq k}^n{(k-i)}} \text{.}
\end{align}   
We have implemented $e_{fd}$ for derivatives up to the third order with the assumption that we can bound 
$\|u^{n+k}\|_{\infty} \leq C_u$ and $\|\xi^{k}\|_{\infty} \leq C_{\xi}$, for $k = 1,..,l$, with some constant $C_u, C_{\xi} >0$.
To classify the uniqueness of a PDE, we consider the ratio of the smallest SV to the largest SV, which is scale-invariant as opposed to S-FRanCo \citep{scholl2023welldefinedness} which only uses the smallest SV.
We define the reciprocal of this condition as $\rho: \mathbb{R}^{m \times n}/ \{0\} \rightarrow [0,1], \rho(A) = \frac{\sigma_n(A)}{\sigma_1(A)}$.
The following corollary provides upper and lower bounds for identifying singular and non-singular matrices, which are used later to identify the uniqueness of PDEs. 
\begin{corollary}
    \label{cor:sing_nonsing_bounds}
    Let $\Tilde{A} = A + E \in \mathbb{R}^{m \times n}$, for $m>n$, with $\|E\|_F\leq \epsilon$ for some $\epsilon > 0$ and $ C_1^{low}\leq \sigma_1(A) \leq C_1^{up}$. If A is singular and $C_1^{low}-\epsilon > 0$,
    \begin{equation}
     \rho(\tilde{A}) \leq \frac{\|E\|_F}{\sigma_1(A)-\|E\|_F} \leq \frac{\epsilon}{ C_1^{low}-\epsilon}.
     \label{eq:singular_matrix_upper_bound}
\end{equation}
 If A is non-singular, i.e., $\sigma_n(A)\geq C_n > 0$ and $C_n-\epsilon > 0$,
\begin{equation}
     \rho(\tilde{A}) \geq \frac{\sigma_n(A) -\|E\|_F}{\sigma_1(A)+\|E\|_F}\geq \frac{C_n-\epsilon}{C_1^{up}+\epsilon}.
     \label{eq:nonsingular_matrix_lower_bound}
\end{equation}    
\end{corollary}
\begin{proof}
    Theorem \ref{theorem:mirksy} implies that 
        $ \|E\|_F + \sigma_k(A) \geq\sigma_k(\Tilde{A}) \geq \|E\|_F - \sigma_k(A)$, for $k\in \{1,...,n\}$.
     If A is singular, i.e., $\sigma_n(A)=0$, and $C_1^{low}-\epsilon > 0$, the previous inequalities imply \ref{eq:singular_matrix_upper_bound}. Similarly, the previous equation implies \ref{eq:nonsingular_matrix_lower_bound}.
\end{proof}
\subsection{Linear PDEs} \label{sec:nr-franco}
Consider a continuously differentiable function $u:\mathbb{R}  \times \mathbb{R}^n\rightarrow \mathbb{R}, (t,x) \mapsto u(t,x) $. In the following we denote the spatial derivatives of $u$ at $x$ as $u_{\alpha}=\partial_{\alpha_1} ... \partial_{\alpha_n}u$, where $\alpha = (\alpha_1, ... , \alpha_n) \in \mathbb{N}^n$. Let $G=(u_{\alpha_1}|...|u_{\alpha_{n}}) \in \mathbb{R}^{m\times n} $ with $m>n$.
Assume, we are given $\tilde{u} = u + \epsilon$ at data points $\{(t_i, x_j)\}_{i, j = 1}^m$ and construct $\tilde{G}=(\tilde{u}_{\alpha_1}|...|\tilde{u}_{\alpha_n})\in \mathbb{R}^{m\times n}$ via central FD.
The bound $e_G$ on the Frobenius norm of $G$
\begin{align}
\label{eq:frobenius_ggtilde_upperbound}
\|G-\tilde{G}\|^2_F \leq\|u-\tilde{u}\|^2_2+\sum_{i=2}^{n} m \|u_{\alpha_i}-\tilde{u}_{\alpha_i}\|_{\infty}^2 \leq \epsilon_{G}
\end{align}
can be computed using $e_{fd}$, since we can bound $\|u_{\alpha_i}-\tilde{u}_{\alpha_i}\|_{\infty}$ with $e_{fd}$.
The following theorem can be used to classify unique and non-unique linear PDEs and is the mathematical foundation for NR-FRanCo.
\begin{theorem}
\label{theorem:sfranco_unique_class}
Let  $C_1^{low}\leq \sigma_1(G) \leq C_1^{up}$ and $\|G-\tilde{G}\|^2_F \leq \epsilon_{G}$.
If a linear PDE is not unique for $G=(u_{\alpha_1},...,u_{\alpha_{n}})$ and $C_1^{low}>\epsilon_{G}$, then
\begin{equation}
    \rho(\Tilde{G}) \leq \frac{\epsilon_{G}}{ C_1^{low}-\epsilon_{G}}
\label{eq:sfranco_nonunique_bound}
\end{equation}
holds. 
If a linear PDE is unique for $G$ and $\sigma_n(G)\geq C_n>\epsilon_{G}$, then the following holds
\begin{equation}
     \rho(\Tilde{G})  \geq \frac{C_n-\epsilon_{G}}{C_1^{up}+\epsilon_{G}}.
\label{eq:sfranco_unique_bound}
\end{equation}
\end{theorem}
\begin{proof}
By \citet{scholl2023welldefinedness} a linear PDE is unique iff the rank of $G$ is full. Then the proof follows directly from Corollary~\ref{cor:sing_nonsing_bounds}.
\end{proof}
Noise robust FRanCo (NR-FRanCo) uses Equations~\ref{eq:sfranco_nonunique_bound} or \ref{eq:sfranco_unique_bound} to exclude either the possibility that the PDE is non-unique or 
unique by computing the threshold and $\rho$. While Theorem~\ref{theorem:sfranco_unique_class} considers deterministically bounded noise, it can be directly extended to probabilistic bounds.
\subsection{Algebraic and analytic PDEs} \label{sec:nr-jrc}
Similar to before, we use $\rho$ to identify the uniqueness of algebraic and analytic PDEs with NR-JRC. We define $G$ and $\tilde{G}$ as in the previous section and bound the noise on the Jacobian of $G$ with
$\|J_G - J_{\tilde{G}}\|^2_F \leq \epsilon_{J_G}$, which can again be directly computed using $e_{fd}$.
The following theorem lays the mathematical groundwork for NR-JRC, by providing thresholds for identifying (non-)uniqueness depending on the error of the Jacobians $e_{J_G}$. 

\begin{theorem}
\label{theorem:jrc_unique_class}
Let  $C_1^{low}\leq \sigma_1(J_G) \leq C_1^{up}$ and $\|J_G-J_{\tilde{G}}\|^2_F \leq \epsilon_{J_G}$.
If an analytic PDE is not unique for a continuously differentiable function $u$ or an algebraic PDE is not unique for an algebraic function $u$, $G=(u_{\alpha_1},...,u_{\alpha_{n}})$, and $C_1^{low}>\epsilon_{J_G}$, then 
\begin{equation}
    \rho(J_{\Tilde{G}}) \leq \frac{\epsilon_{J_G}}{ C_1^{low}-\epsilon_{J_G}}
\label{eq:jrc_nonunique_bound}  
\end{equation}
holds for all data points.
If an algebraic PDE is unique for an algebraic function $u$, $G$ as above, and $\sigma_n(J_G)\geq C_n>\epsilon_{J_G}$ then for at least one data point it holds that
\begin{equation}
     \rho(\Tilde{J_G})  \geq \frac{C_n-\epsilon_{J_G}}{C_1^{up}+\epsilon_{J_G}}.
\label{eq:jrc_unique_bound}  
\end{equation}
\end{theorem}
\begin{proof}
By \citet{scholl2023welldefinedness} if the Jacobian $J_G$ has full rank for at least one data point the analytic PDE is unique. Corollary \ref{cor:sing_nonsing_bounds} implies Equation \ref{eq:jrc_unique_bound}.
If the Jacobian does not have full rank for all data points and the PDE is algebraic, the PDE is not unique \citep{scholl2023welldefinedness}. Consequently, Corollary \ref{cor:sing_nonsing_bounds} implies Equation \ref{eq:jrc_nonunique_bound}.
\end{proof}

Note that the implication only holds for one direction for analytic PDEs as proven in \citet{scholl2023welldefinedness}. As mentioned for Theorem~\ref{theorem:sfranco_unique_class}, Theorem~\ref{theorem:jrc_unique_class} can also be directly extended to probabilistically bounded noise. Noise robust JRC (NR-JRC) applies the thresholds provided in Equations~\ref{eq:jrc_nonunique_bound} and \ref{eq:jrc_unique_bound} to decide if an algebraic/analytic PDE is unique or not unique, similarly to NR-FRanCo. 


\section{Experiments}
\label{sec:experiments}
In this section, we present the empirical results of NR-FRanCo and NR-JRC in various, simulated experiments, in which we add Gaussian noise with standard deviation $\alpha^2\|u\|_2^2$, where $\alpha$ is the noise level. 
In all experiments we fix $G=(u,\partial_x u)$ and denote 
with $\tilde{\sigma}_i$ the SVs of $\tilde{G}$ for NR-FRanCo
and of $J_{\tilde{G}}$ for NR-JRC.
Furthermore, we set the parameters $C = 10 ^{-4}$, $C_1^{low}= 0.5 \tilde{\sigma}_1$, $C_1^{up}= 1.5 \tilde{\sigma}_1$, and $C_n=\max(C\tilde{\sigma}_1,0.5\tilde{\sigma}_n)$ throughout the experiments.
Moreover, we compute $C_u$ by FD as $C_u = \max_{k = 1,..,l}(\|\tilde{u}^{n+k}\|_{\infty})$ and assume $C_{\xi} = 1$.\footnote{The code can be found in \href{https://github.com/HillaryHauger/robust-identifiability-symbolic-recovery}{https://github.com/HillaryHauger/robust-identifiability-symbolic-recovery}.} Table \ref{tab:linear_pdes} and \ref{tab:jrc_pdes} specify the PDEs used in the experiments.
\begin{table}[]
\small
    \centering
    \begin{tabular}{c|c|c}
            & Function & PDE \\
         a) & $\exp(ax+t)$ & $\partial_t u=au=a\partial_x u$ \\
         b) & $\cos(x-at)$ & $\partial_t u=-a\partial_x u$    \\
         c) & $\sin(x-at)$ & $\partial_t u=-a\partial_x u$ \\
         d) & $(x+bt)\exp(at)$ & $\partial_t u=au+b\partial_x u$ \\
    \end{tabular}
    \caption{Linear PDEs used for performing experiments with NR-FRanCo.}
    \label{tab:linear_pdes}
\end{table}

In Figure \ref{fig:nr_franco_results_one} one can see the results of NR-FRanCo for the non-unique PDE in Table \ref{tab:linear_pdes}a. The classification works well for noise levels $0$ and $10^{-8}$. Furthermore, it is visible that with higher FD order the hypotheses become more certain.
Figure~\ref{fig:nr_franco_results_all} showcases the results of NR-FRanCo for different noise levels and FD orders for multiple PDEs. This shows that classifying a unique PDE with noise is much more robust than classifying a non-unique PDE with noise. 
In Figure \ref{fig:nr_franco_results_all} the classification for FD orders strictly bigger than 4 works until noise level $10^{-5}$. For uniqueness, the classification does not fail until noise level $10^{-1}$.

\begin{figure}[htb]
\centering
\begin{subfigure}[b]{\textwidth}
  \centering
  \centerline{\includegraphics[width=0.5\textwidth]{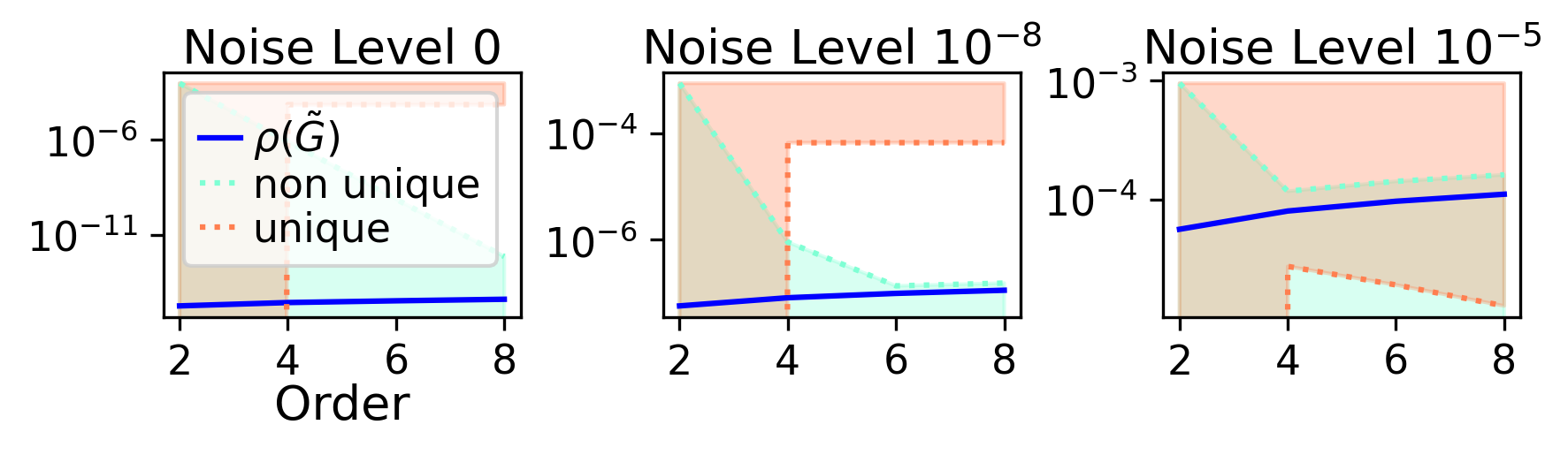}}
  \caption{NR-FRancO results for the PDE from Table \ref{tab:linear_pdes}a).}
  \label{fig:nr_franco_results_one}
\end{subfigure}
\centering
\begin{subfigure}[b]{\textwidth}
\centering
  \centerline{\includegraphics[width=0.5\textwidth]{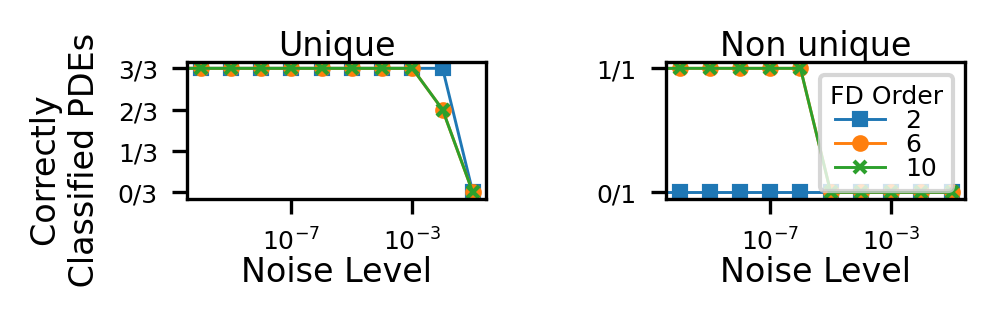}}
  \caption{NR-FRancO results averaged over all PDEs from Table \ref{tab:linear_pdes}.}
  \label{fig:nr_franco_results_all}
\end{subfigure}
\caption{ Figure \ref{fig:nr_franco_results_one} shows NR-FRanCo results for a non-unique PDE and three noise levels. The blue line is $\rho(\tilde{G})$ for FD orders 2, 4, 6, and 8. Red and green lines are bounds from Theorem \ref{theorem:sfranco_unique_class}. $\rho(\tilde{G})$ in the red segment indicates uniqueness, in the green segment non-uniqueness, and in both segments means no definitive conclusion. Figure \ref{fig:nr_franco_results_all} shows how many PDEs were correctly classified by NR-FRanCo for different orders and noise levels $0,10^{-10}, ..., 10^{-1}$ for the PDEs in Table \ref{tab:linear_pdes}.}
\end{figure}


\begin{table}
    \centering
    
    \begin{tabular}{c|c|c}
            & Function & PDE \\
         a) & $(x+t)^{-1}$ & $\partial_t u=\partial_x u=-u^2$ \\
        b)& $(x+t)^{-\frac{1}{2}}$ & $\partial_t u=\partial_x u=-\frac{1}{2}u^2$\\
         c) & $\left(x + t\right) \arccos\left( \cosh(a t)^{-1} \right)$ & $\partial_t u=\partial_x u+\frac{u}{\partial_x u}\cos(u)$  \\
         d) & $\left(x + t\right)  \arcsin\left(\cosh(a t)^{-1}\right)
$ & $\partial_t u=\partial_x u-\frac{u}{\partial_x u}\sin(u)$   \\
    \end{tabular}
    \caption{Algebraic and analytic PDEs used for experiments with NR-JRC.}
    \label{tab:jrc_pdes}
\end{table}
Figure \ref{fig:nr_jrc_results_one} presents the results of NR-JRC for a non-unique algebraic PDE  with FD order 8. It shows that the classification of non-unique PDEs is very unstable and already fails at the low noise level of $10^{-8}$. In Figure \ref{fig:nr_jrc_results_all} the results for NR-JRC applied to several different PDEs are visualized. The uniqueness classification is more robust than the non-uniqueness classification: it classifies all unique PDEs correctly up to noise level $10^{-6}$.
The non-unique classification works up till noise level $10^{-9}$ for both PDEs if the FD order is at least 4.
\begin{figure}[htb]
\begin{minipage}[b]{1.0\textwidth}
  \centering
  \centerline{\includegraphics[width=0.5\textwidth]{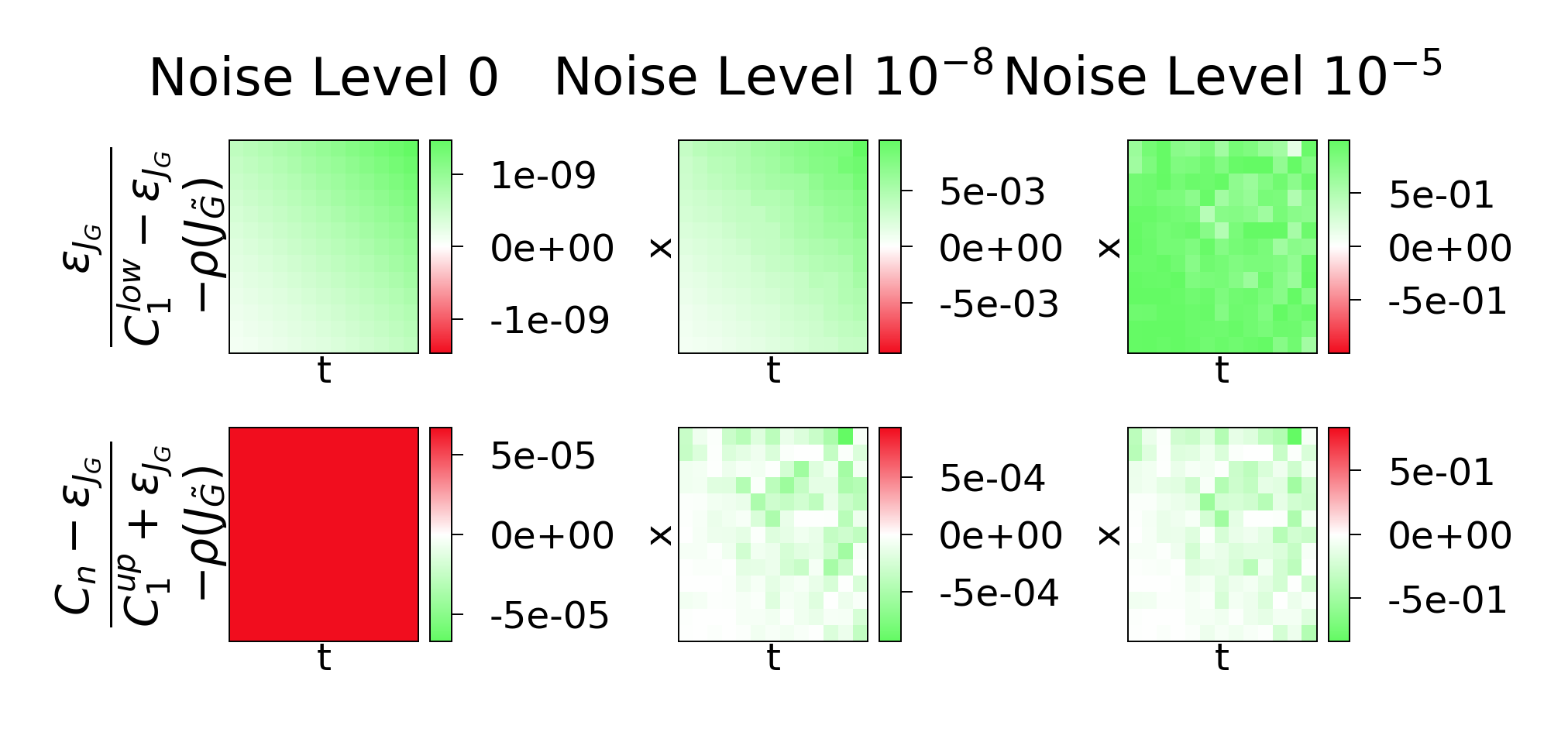}}
  \subcaption{JRC results for the PDE from Table \ref{tab:jrc_pdes}a).}
  \label{fig:nr_jrc_results_one}
\end{minipage}
\begin{minipage}[b]{1.0\textwidth}
  \centering
  \centerline{\includegraphics[width=0.5\textwidth]{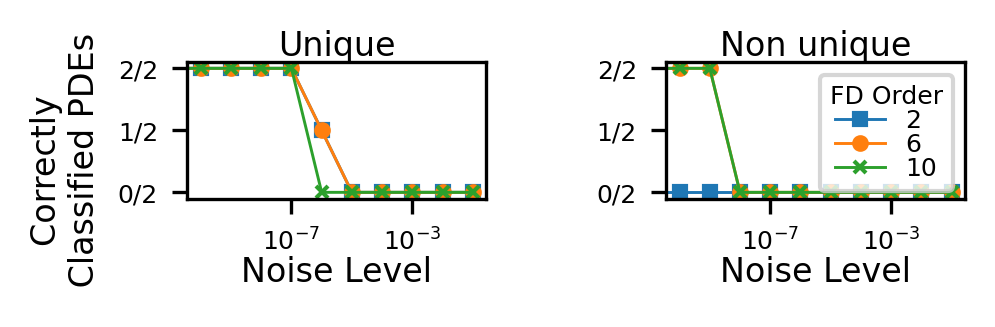}}
  \subcaption{JRC results averaged over all PDEs from Table \ref{tab:jrc_pdes}.}
    \label{fig:nr_jrc_results_all}
\end{minipage}%
\caption{Figure \ref{fig:nr_jrc_results_one} shows NR-JRC results for three noise levels on the PDE in Table \ref{tab:jrc_pdes}a). It subtract $\rho(J_{\tilde{G}})$ from the thresholds from Theorem~\ref{theorem:jrc_unique_class}. The PDE is non-unique if the upper plot is green and the lower plot is red, and vice versa for unique PDEs. If both plots have the same color, no conclusion can be made. Figure \ref{fig:nr_jrc_results_all} shows how many PDEs were correctly classified by the NR-JRC for different orders and noise levels $0,10^{-10}, ..., 10^{-1}$ for all PDEs specified in Table \ref{tab:jrc_pdes}.}
\end{figure}


\section{Discussion}
Interestingly, the NR-JRC method is significantly less stable than NR-FRanCo. This instability arises because NR-JRC requires calculating more derivatives, as it uses the Jacobian of $G$. Consequently, FD introduces much larger errors. These errors can severely impact the classification of the uniqueness of PDEs. 
Furthermore, one major conclusion is that the classification of unique PDEs is much more robust than of non-unique PDEs. This stems from the fact, that the smallest SV is very sensitive to small noise perturbations for non-uniqueness, as shown in Theorem \ref{theorem:mirksy}. 
Another crucial aspect is that while one often cannot determine the uniqueness of the solution, this issue is not necessarily a flaw in the algorithm but rather a reflection of the problem's ill-posedness. In certain cases, for instance, when the noise level is too high, one can not decide mathematically whether the PDE is unique or not.
One limitation of our method is that the bounds specified in Theorems \ref{theorem:sfranco_unique_class} and \ref{theorem:jrc_unique_class}, depend on constants possibly unknown in practice ($C_1^{low}$, $C_1^{up}$, $C_n$, $C_u$ and $C_{\xi}$), which have to be estimated.
\section{Conclusion}
\label{sec:conlusion}
We prove theoretical bounds to determine the (non-)uniqueness of PDEs and propose methods for identifying the uniqueness of PDEs in the presence of noise based on these results. These methods can also identify situations where a decision cannot be made due to excessive noise. We show empirically that they can be successfully used on examples for which the previous methods would have failed. 

\bibliographystyle{plainnat}  
\bibliography{references}  

\end{document}